\documentclass{article}
\usepackage{authblk}
\usepackage{arxiv}

\usepackage[utf8]{inputenc} 
\usepackage[T1]{fontenc}    
\usepackage{hyperref}       
\usepackage{url}            
\usepackage{booktabs}       
\usepackage{amsfonts}       
\usepackage{nicefrac}       
\usepackage{microtype}      
\usepackage{lipsum}

\usepackage{graphics} 
\usepackage{graphicx}
\usepackage{amsmath} 
\usepackage{amssymb}  
\usepackage{pgfplots}
\usetikzlibrary{arrows.meta}

\newcommand{\rc}[1]{\textcolor{black}{#1}}

\title{\LARGE \bf
 Move-to-Data: A new  Continual Learning approach with Deep CNNs, Application for image-class recognition
 \thanks{This work was supported by ERASMUS+ Internship Mobility grant and CNRS Interdisciplinary Grant "RoBioVis". }
}

\author[1]{Miltiadis Poursanidis}
\author[1]{Jenny Benois-Pineau}
\author[1]{Akka Zemmari} 
\author[1]{Boris Mansenca} 
\author[2]{Aymar de Rugy}
\affil[1]{LaBRI, Univ. Bordeaux - CNRS, Bordeaux, France.}
\affil[2]{INCIA, Univ. Bordeaux - CNRS, Bordeaux, France.}

\begin{document}

\maketitle
\thispagestyle{empty}
\pagestyle{empty}

\begin{abstract}
In many real-life tasks of application of supervised learning approaches, all the training data are not available at the same time. The examples are life-long image classification or recognition of environmental objects during interaction of instrumented persons with their environment, enrichment of an online-database with more images.  It is necessary to pre-train the model at a "training recording phase" and then adjust it to the new coming data. This is the task of incremental/continual learning approaches. Amongst different problems to be solved by these approaches such as introduction of new categories in the model, refining existing categories to sub-categories and extending trained classifiers over them, ... we focus on the problem of adjusting pre-trained model with new additional training data for existing categories. We propose a fast continual learning layer at the end of the neuronal network. Obtained results are illustrated on the opensource CIFAR benchmark dataset. The proposed scheme yields similar performances as retraining but with drastically lower computational cost.
\end{abstract}

\section{INTRODUCTION AND RELATED WORK}
\label{sec:intro}

Deep Learning has gained an increasing attention during the past years \cite{Lecun2015}, specifically Convolutional Neural Networks (CNNs) for different visual recognition tasks \cite{GU2018354}. Hence, the problems which have been solved by the past with other machine learning approaches have to be developed from Deep Learning perspective. One of them is the incremental/continual/life-long learning. With life-long learning, we mean the ability of humans to learn through experience overtime. For supervised (Deep Learning) approaches, it means that the data for model training are not all available at glance and the model has to be adjusted with the new coming data. In the case of unsupervised learning, such approaches have been proposed since quite a time, e.g. incremental clustering \cite{DBLP:journals/pr/Lughofer08},\cite{ZhangFS08}, \cite{DBLP:conf/cbmi/MansencalBVD12}. In supervised (Deep Neural Networks) framework, while offline models have shown to be successful in an abundance of fields, such as medical image classification \cite{AderghalKKBAC18}, recognition of historical objects in digital cultural heritage management \cite{ObesoBGR18} and others, online models still seem to lack effectiveness.\\
The scenarios which might require continual learning approaches are the following but not limited to: 
i) introduction of new categories from new coming data \cite{DBLP:conf/eccv/CastroMGSA18}, specifically with a few data \cite{DBLP:conf/cvpr/GidarisK18}; ii) refining existing categories to sub categories (hierarchical classification) \cite{DBLP:conf/cbmi/DuttPQ17}; iii) additional training data for existing categories. The latter is not only needed in the scenarios where the existing database is permanently enriched (such as in cultural heritage management), but also in tracking of objects of known class \cite{BMVC.28.56}, which is a "online-learning" for class-specific tracking considered in this work, requiring re-training with gradient descent. Our contribution concerns this very case: once a model has been pre-trained on a sufficiently large data set, it has to be adjusted with new coming data without adding or refining categories. \\
Here comes an important question about catastrophic forgetting. It means that if a system has been trained on a bunch of data and shows good performances, then sequentially trained with new data on new tasks, it could forget about the older tasks \cite{DBLP:journals/corr/CatastrophForget}. This is the so-called "stability-plasticity dilemma". 

In case of human reasoning, the stability–plasticity dilemma regards the extent to which a system must be prone to adapt to new knowledge and, importantly, how this adaptation process should be compensated by internal mechanisms that stabilise and modulate neural activity to prevent catastrophic forgetting \cite{Parisi2019}. 
The catastrophic forgetting is already present in the neural networks due to the process of gradient-descent based optimisation of the parameters. 
At each forward pass the loss is computed with the available model, then during the backward propagation, the gradient of the loss with respect to each parameter is computed. And the parameter adjustment during the optimization process is performed with this new loss \cite{Rumelhart'1986}. Therefore, if the training data is coming sequentially, as this is the case in continual learning, we will observe a « model drift » from already optimized solution. Some attempts have been made to avoid catastrophic forgetting as in \cite{ShmelkovSA17}, but the authors consider the case of new classes which appear without pre-training. And their problem is to balance the performances on "old" classes and new ones. We are interested in the case when the taxonomy does not change along the time, but the object appearance may do, as in tracking due to auto occlusions and progressive changes of the view-point or in database enrichment, when the initial database is continuously incremented by different view of the same visual content. 
The re-training of neural networks in order to adapt models to the new coming content requires a heavy computational workload due to the back-propagation pass in gradient optimization.   

In this paper, we propose  a novel approach for continual learning, which does not require gradient based optimization.
The motivation of it was the recognition of objects to grasp in assistance to amputees with vision-based Neuro-prostheses \cite{Gonzalez-DiazBD19},\cite{Gonzalez-Diaz2018} when different views of the same object -to-grasp come "on the fly" to adjust the pre-trained model. Another scenario is continual image database enrichment as in the application of \cite{ObesoBGR18}. We show that the method performs "not worse" than continual learning by sequential gradient descent optimization. A mathematical fromulation of the method is given and experiments  on open image dataset CIFAR-10 \cite{CIFAR2009} are reported. 
Our approach to incremental learning differs from all the others introduced so far. During the learning procedure we aim to internally change the neuronal networks weight structure rather than changing the whole architecture. We also distance ourselves from classical retraining, which is too computationally expensive for real world applications. When processing data on the fly retraining would be far from real time. 
The core idea of the method is the adjustment of a weight of the neuron responding to the class of the training example coming sequentially "on the fly". The reminder of the paper is organized as follows. In Section 2 we present mathematical bases of our method. In Section 3 we report on the experiments on publicly available benchmark database. Section 4 concludes this work and outlines its perspectives.  






\section{MOVE-TO-DATA: A CONTINUAL LEARNING  METHOD}
\label{sec:MOVE_TO_DATA}
In the following we will restrict ourselves to the context of deep learning even thought the definitions can be expanded to the field of machine learning in general. To explain Move-to-Data method we will first introduce notations. 
\subsection{Definitions and Notations}
Suppose we have a sequence of labeled data $(x_i, y_i)_{i \in I}$ with some index set $I\subset \mathbb{N}$. Let us also assume we have $d$-dimensional input data $x_i \in \mathbb{R}^d$ and $c$-dimensional 1-of-c encoded label $y_i \in \{e_1, \dots , e_c \}$, where $e_i\in \mathbb{R}^c$ is a unit vector while $c\in \mathbb{N}$ is the number of classes. To given parameters $\Theta \in \mathbb{R}^N$, \rc{where $N$ is the number of the neural network parameters}, we \rc{can} define the \rc{neural network} as a function, see eq.:
\begin{equation}
\label{eq:GeneralFunction}
f_\Theta : \mathbb{R}^{\rc{d}} \to \rc{\{e_1, \dots , e_c \}^c}.    
\end{equation}

Assume further that the activation functions of our neural network are: 
\begin{equation}
\label{eq:ReLu}
Relu(x) = \max\{0,x\},    
\end{equation}
Any other non-linear response can be used such as sigmoid, leaky ReLu, ... Please note that we implicitly assumed a fixed architecture of the Neuronal Network, that is a known shape of the function (\ref{eq:GeneralFunction}).

Let $T\in \mathbb{N}$ and $S:= \{ (x_i,y_i): i \in \mathbb{N} \}$. Suppose now we have a neural network $f_{\Theta_T}$ trained offline on the data:
\begin{equation}
\label{eq:TrainedNetwork}
S_{< T} := \{ (x_i,y_i) \in S : i \in I, i <T  \}    
\end{equation}
In the online learning setting, the learning step would consist of finding the parameters $\Theta_T$ by using all previous information, that is, all previous models and all previous data. We are thus searching for a mapping:
\begin{equation}
\label{eq:AllOnlineMapping}
  (\Theta_1 , \Theta_2,  \dots \Theta_{T-1}, S_{< \rc{T+1}}) \mapsto \Theta_{T}  
\end{equation}
In many applications, we neither have the storage nor the computational power to process all past models $\Theta_1 , \Theta_2,  \dots \Theta_{T-1}$ and data $S_{< \rc{T+1}}$. Therefore, in incremental learning, we would like to find a mapping of the form:
\begin{equation}
\label{eq:IncrementalLearning}
 \Theta_{T-1}, (x_T, y_T) \mapsto \Theta_{T}   
\end{equation}
depending only on the first new coming data $(x_T, y_T)$  and on the previous model $\Theta_{T}$. 

\subsection{"Move-to-Data": Incremental learning approach for a Deep CNN}
\label{subse:Move-To-data}
The leading idea of our approach is to maintain the overall structure of the deep neural network and slightly adapt the model to the new data stream on a small scale. 

The naive approach would be to retrain the model with gradient descent. Nevertheless, it is time consuming to apply back propagation on each new data point arriving on the fly. 

Let us start off with an observation. Let $v, w\in \mathbb{R}^d$ with $ \| v \| , \| w \| = 1$. Then the scalar product $ \langle u, w \rangle  $ is large if the angle between $v$ and $w$ is small. This, rather trivial observation, indicates that for a high activation at a neuron, the weight vector and the feature vector must have been similar regarding the angle between them. This is often referred as the cosine similarity, this principle has been largely used in Content-Based Image Retrieval Systems \cite{Hafner'95}. 

In CNNs convolutional layers serve to extract features and are followed by fully connected (FC) implementing a neural classifier, see for instance \cite{ZemmariB20} for explanations. These layers which can be several cascaded,are implementing Multi-Layerd Perceptron with hidden layers or not. Let us now focus on one last hidden (FC) layer $L$ of a neural network  with width $l \in \mathbb{N}$. As above, we have a feature vector $v_i \in \mathbb{R}^l$ corresponding to the input data $x_i$ for some $i\in \mathbb{N}$. The output vector $ \hat y_i$ is then given by $\hat y_i = W v_i$, where the weights $W$ are defined as follows:

\begin{equation}
\label{eq:weights}
W = (w_1, \dots, w_l)^{T} \mbox{ and } w_j \in \mathbb{R}^c, \forall j = 1, \dots, l.
\end{equation}
For the activation of the $j$-th neuron corresponding to the $j$-th class, we have $\hat y_i^j =\langle w_j, v_i \rangle$. As suggested above, the activation $\hat y_i^j$ will be increased if the weight vector $w_j$ is closer to the feature vector $v_i$. So for given label in the incremental step $y_i$ belonging to the class $j$ for some $j \in \{1, \dots c \}$, in other words $y_i = e_j$, we move the weight vector $w_j$ to the direction of the feature vector $v_i$ as defined in the following equation:

\begin{eqnarray}
\label{eqnarray::update}
w'_j = w_j +(v_i-w_j)\epsilon,
\end{eqnarray}
where $1 > \epsilon > 0 $ is chosen be small. 

In the context of incremental learning, we are receiving new data $(x_T, y_T)$ on the fly. For each new data point we apply the formula (\ref{eqnarray::update}). This procedure we call \textit{Move-to-Data}. The Move-to-Data method  lets  
the loss function to decrease (see section \ref{sec:ExperimentsAndResults)}. However, one should notice that the loss function will continue decreasing until all the classes have been seen in the samples.

Furthermore, it is important to note that as formulated in (\ref{eqnarray::update}) $w_j$ and $v_i$ need to be unit vectors to prevent biases induced by scaling. It is common practice to normalize feature and weight vectors. For features it is better known as "feature scaling" and for weights it has the name of re-parametrization\cite{DBLP:journals/corr/SalimansK16}. In our case, we do not normalize the weights $w_j$, but move them to the data vector by the following equation using projection of weigt vector $w_j$ on data vector $v_i$ direction: 

\begin{eqnarray}
\label{eq:ProjectionUpdate}
w'_j =  w_j +(\left\lVert w_j \right\rVert*\frac{v_i}{\left\lVert v_i\right\rVert}-w_j)\epsilon,
\end{eqnarray}

where $\left\lVert .\right\rVert$ is the Euclidean norm. 
Note that the proposed adjustment of weights concerns only FC layer. Hence it is applicable not only to CNNs but to a classical MLP as well as to recursive neural networks.

\section{EXPERIMENTS AND RESULTS}
\label{sec:ExperimentsAndResults)}
In this section we describe experiments we have conducted to evaluate proposed Move-to-Data approach on an open dataset: CIFAR10 \cite{CIFAR2009}.


The CIFAR10 dataset consists of 60000 RGB images  of dimension 32x32x3. There are 10 classes with 6000 images per class, respectively. The classes are mutually exclusive. The dataset is split up into 50000 training images and 10000 test images which we will call "original training set" and original "test set". 


\subsection{CNN configuration}

We use a ResNet56v2 \cite{he2016identity} convolutional neural network. 
It is trained with a batch size of 32, for 160 epochs from scratch on a subset of the 50000- images original training set of  CIFAR10, see \ref{subsec:DatasetUsage}.
The optimizer is Adam \cite{kingma2014adam}, with  decreasing rate from 0.001 to 0.0001 accordingly to the following non-linear function eq;
\begin{equation}
\label{eq:LRDecay}
    lr_i = lr_{i-1} + \frac{1}{1+\delta*i}
\end{equation}
Here $lr$ denotes learning rate,  $\delta$ is the $lr$ decay, $\delta = 10^{-6}$, and $i$ is the number of iteration. The loss function is the cross entropy loss (see \cite{ZemmariB20} for a formal definition). 
Slight data augmentation (random translations and horizontal flip) is used during training.


\subsection{Dataset usage}
\label{subsec:DatasetUsage}
First, a model is trained on a subset of the training images from original training set.
This training is done only on 10\% (i.e., 5000 images) of the original training set of 50000 images. The images are  almost uniformly distributed between classes: the less populated class contains 460 images that is 9.2\% of the offline training subset and the most populated class contains 520 images (10.4\%). The purpose of this split was twofold: first, to have enough (90\%) of images remaining to evaluate the incremental learning. Second, we have a less precise model to better distinguish changes in accuracy during incremental learning.


The remaining 90\% (i.e., 45000) images are split in $N$ data chunks. In this work we used N = 10, thus having 4500 images in each data chunk. We consider that these chunks are received progressively, one after another. The models are progressively adapted on all images at each data chunk reception. The accuracy metric is computed once all chunk of 4500 has been used for model adaptation. 

\subsection{Results}

We compare two methods for quick model adaptation to each newly received chunk. The comparison is done on the "original test set"  The base-line method, is a plain fine-tuning applied only on the last FC layer using a new data-chunk. This means that all parameters in convolution layers are frozen as it is done for some layers in \cite{Yosin'2014}. The second method is our Move-to-Data method. 

For fair comparison, the fine-tuning is done with a batch size of 1. This is the same strategy as in Move-to-Data, when the weight vector $W$ is adjusted with each coming and passed through the network data vector. Both methods are implemented in Keras \cite{chollet2017deep}, with a tensorflow backend, and are run on CPU.

The Figure~\ref{fig:MtdEpsilon} shows Move-to-Data models' accuracies for different choices of $\epsilon$ parameter (see Eq. \eqref{eqnarray::update}). 

\begin{figure}[!h]
     \centering
     \includegraphics[width=0.5\textwidth]{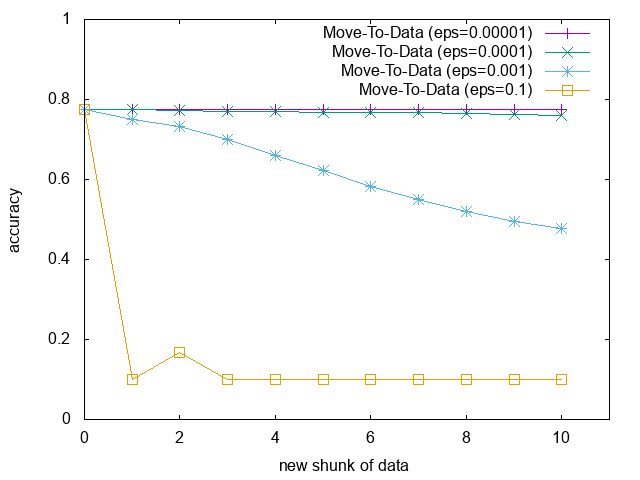}
     \caption{Accuracy for different values of epsilon for Move-to-Data model.}
     \label{fig:MtdEpsilon}
 \end{figure}

Clearly, too strong "move" with $\epsilon$ of 0.1 yields very strong model drift with a catastrophic forgetting. The model does not generalise on the data. If the "Move-to-data" is weak $\epsilon = 0.0001$ then the Move-to-Data method gives decent accuracies close to 0.78 with slight increase over chunks of arriving data. 

The Figure~\ref{fig:FtVsMtd} shows the evolution of model accuracy for successive data chunks for both methods: Move-to-Data and classical fine-tuning, here the $\epsilon$ parameter is fixed to 0.0001. The accuracies are very close, starting form initial accuracy 0.776.  at the beginning Move-to-Data is even slightly better and at the end, at 10th chunk, they are practically equal (0.762 for fine tuning and 0.761 for Move-to-Data). This means that our method has the same "catastrophic forgetting" as the gradient descent sequential fine-tuning. And this is without heavy gradient descent computations.  

\begin{figure}[!h]
     \centering
     \includegraphics[width=0.5\textwidth]{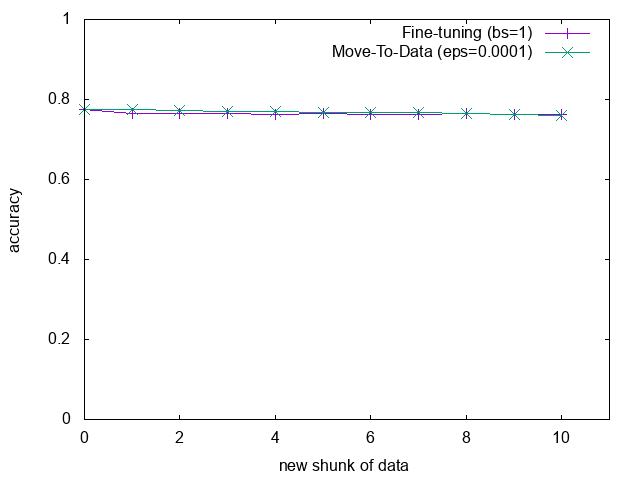}
     \caption{Accuracy between successive data chunks (N=10) for the two methods: Fine-tuning and Move-to-Data.}
     \label{fig:FtVsMtd}
 \end{figure}
Our method was implemented only in CPU. Hence to compare its computational time to the fine-tunning base-line, we also perform the fine tuning only in CPU. The results are illustrated in  Figure~\ref{fig:FtVsMtdTimeGPU}. We also give Fine-Tuning times with GPU acceleration with batch size 1 (as in our case of the base-line).  Obviously, the CPU implementation cannot compete with GPU acceleration, but all conditions equal (CPU), the Move-to-Data largely bypasses fine-tuning in computational speed, being more than  4 times faster than the latter. The mean computation times along the chunks of the data  are 1327,179 $\pm$136,908 and 333,304$\pm$47,499 for fine-tuning and Move-to-Data respectively for CPU implementation.


\begin{figure}[!h]
     \centering
     \includegraphics[width=0.5\textwidth]{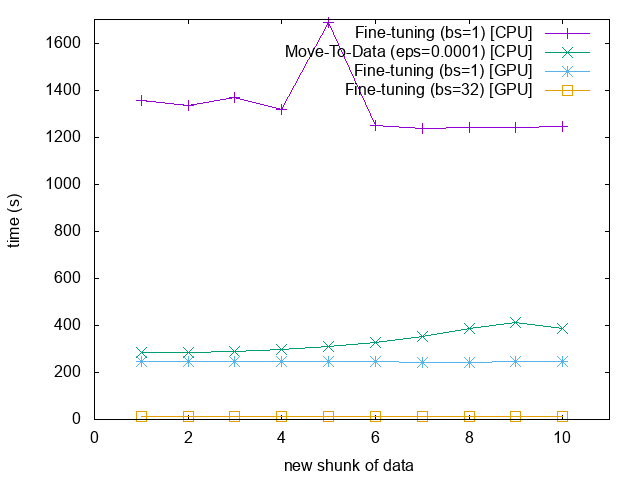}
     \caption{Time of various methods, on CPU or GPU.}
     \label{fig:FtVsMtdTimeGPU}
 \end{figure}

\section{CONCLUSION AND PERSPECTIVES}
Hence, in this paper, we have proposed a new method, Move-to-Data, which is a continual learning approach for deep convolutional neural networks classifiers. We presented and discussed mathematical formulation of the approach and tested it on a publicly available dataset CIFAR10. The method acts only on the last layer of the last fully connected layer of a "classification" part of a CNN. It is generic and can be applied to other kinds of Neural Networks: MLP and RNNs. The experiments show that it is more than 4 times faster than the base-line fine-tuning with the gradient descent while having the same catastrophic forgetting effect measured by comparison of accuracies attained by the two methods.  The next step would be to extend the Move-to-Data method to the last two or three fully connected layers. 

This paper introduces the method and presents the results of its experimental evaluation.  The proof of convergence remains in the perspective of this work as well as its application for objects tracking in video.









\section*{Acknowledgments}
The work has been supported by CNRS Interdisciplinary Grant RoBioVis, ERASMUS + internship grant, University of Bordeaux. We thank master student Eliot Ragueneau for his help in data preparation. 




\bibliographystyle{IEEEtran}
\bibliography{Incremental_Learning.bib}

\begin{thebibliography}{10}
\providecommand{\url}[1]{#1}
\csname url@samestyle\endcsname
\providecommand{\newblock}{\relax}
\providecommand{\bibinfo}[2]{#2}
\providecommand{\BIBentrySTDinterwordspacing}{\spaceskip=0pt\relax}
\providecommand{\BIBentryALTinterwordstretchfactor}{4}
\providecommand{\BIBentryALTinterwordspacing}{\spaceskip=\fontdimen2\font plus
\BIBentryALTinterwordstretchfactor\fontdimen3\font minus
  \fontdimen4\font\relax}
\providecommand{\BIBforeignlanguage}[2]{{%
\expandafter\ifx\csname l@#1\endcsname\relax
\typeout{** WARNING: IEEEtran.bst: No hyphenation pattern has been}%
\typeout{** loaded for the language `#1'. Using the pattern for}%
\typeout{** the default language instead.}%
\else
\language=\csname l@#1\endcsname
\fi
#2}}
\providecommand{\BIBdecl}{\relax}
\BIBdecl

\bibitem{Lecun2015}
Y.~Lecun, Y.~Bengio, and G.~Hinton, ``{Deep learning},'' pp. 436--444, 2015.

\bibitem{GU2018354}
\BIBentryALTinterwordspacing
J.~Gu, Z.~Wang, J.~Kuen, L.~Ma, A.~Shahroudy, B.~Shuai, T.~Liu, X.~Wang,
  G.~Wang, J.~Cai, and T.~Chen, ``Recent advances in convolutional neural
  networks,'' \emph{Pattern Recognition}, vol.~77, pp. 354 -- 377, 2018.
  [Online]. Available:
  \url{http://www.sciencedirect.com/science/article/pii/S0031320317304120}
\BIBentrySTDinterwordspacing

\bibitem{DBLP:journals/pr/Lughofer08}
\BIBentryALTinterwordspacing
E.~Lughofer, ``Extensions of vector quantization for incremental clustering,''
  \emph{Pattern Recognition}, vol.~41, no.~3, pp. 995--1011, 2008. [Online].
  Available: \url{https://doi.org/10.1016/j.patcog.2007.07.019}
\BIBentrySTDinterwordspacing

\bibitem{ZhangFS08}
X.~Zhang, C.~Furtlehner, and M.~Sebag, ``Distributed and incremental clustering
  based on weighted affinity propagation,'' in \emph{{STAIRS}}, ser. Frontiers
  in Artificial Intelligence and Applications, vol. 179.\hskip 1em plus 0.5em
  minus 0.4em\relax {IOS} Press, 2008, pp. 199--210.

\bibitem{DBLP:conf/cbmi/MansencalBVD12}
\BIBentryALTinterwordspacing
B.~Mansencal, J.~Benois{-}Pineau, R.~Vieux, and J.~Domenger, ``Search of
  objects of interest in videos,'' in \emph{10th International Workshop on
  Content-Based Multimedia Indexing, {CBMI} 2012, Annecy, France, June 27-29,
  2012}, 2012, pp. 1--6. [Online]. Available:
  \url{https://doi.org/10.1109/CBMI.2012.6269809}
\BIBentrySTDinterwordspacing

\bibitem{AderghalKKBAC18}
K.~Aderghal, A.~Khvostikov, A.~Krylov, J.~Benois{-}Pineau, K.~Afdel, and
  G.~Catheline, ``Classification of alzheimer disease on imaging modalities
  with deep cnns using cross-modal transfer learning,'' in \emph{{CBMS}}.\hskip
  1em plus 0.5em minus 0.4em\relax {IEEE} Computer Society, 2018, pp. 345--350.

\bibitem{ObesoBGR18}
A.~M. Obeso, J.~Benois{-}Pineau, M.~S. Garc{\'{\i}}a{-}V{\'{a}}zquez, and A.~A.
  Ram{\'{\i}}rez{-}Acosta, ``Introduction of explicit visual saliency in
  training of deep cnns: Application to architectural styles classification,''
  in \emph{{CBMI}}.\hskip 1em plus 0.5em minus 0.4em\relax {IEEE}, 2018, pp.
  1--5.

\bibitem{DBLP:conf/eccv/CastroMGSA18}
F.~M. Castro, M.~J. Mar{\'{\i}}n{-}Jim{\'{e}}nez, N.~Guil, C.~Schmid, and
  K.~Alahari, ``End-to-end incremental learning,'' in \emph{{ECCV} {(12)}},
  ser. Lecture Notes in Computer Science, vol. 11216.\hskip 1em plus 0.5em
  minus 0.4em\relax Springer, 2018, pp. 241--257.

\bibitem{DBLP:conf/cvpr/GidarisK18}
S.~Gidaris and N.~Komodakis, ``Dynamic few-shot visual learning without
  forgetting,'' in \emph{{CVPR}}.\hskip 1em plus 0.5em minus 0.4em\relax {IEEE}
  Computer Society, 2018, pp. 4367--4375.

\bibitem{DBLP:conf/cbmi/DuttPQ17}
A.~Dutt, D.~Pellerin, and G.~Qu{\'{e}}not, ``Improving hierarchical image
  classification with merged {CNN} architectures,'' in \emph{{CBMI}}.\hskip 1em
  plus 0.5em minus 0.4em\relax {ACM}, 2017, pp. 31:1--31:7.

\bibitem{BMVC.28.56}
H.~Li, Y.~Li, and F.~Porikli, ``Deeptrack: Learning discriminative feature
  representations by convolutional neural networks for visual tracking,'' in
  \emph{Proceedings of the British Machine Vision Conference}.\hskip 1em plus
  0.5em minus 0.4em\relax BMVA Press, 2014.

\bibitem{DBLP:journals/corr/CatastrophForget}
A.~Mallya and S.~Lazebnik, ``Piggyback: Adding multiple tasks to a single,
  fixed network by learning to mask,'' \emph{CoRR}, vol. abs/1801.06519, 2018.

\bibitem{Parisi2019}
G.~I. Parisi, R.~Kemker, J.~L. Part, C.~Kanan, S.~Wermter, and K.~Technology,
  ``{Continual Lifelong Learning with Neural Networks : A Review},'' pp. 1--29,
  2019.

\bibitem{Rumelhart'1986}
G.~E.~H. David E~Rumelhart and R.~J. Williams, ``Learning internal
  representations by error propagation,'' in \emph{Exploration sin the
  Microstructure of Cognition}, 1986, pp. 318--362.

\bibitem{ShmelkovSA17}
K.~Shmelkov, C.~Schmid, and K.~Alahari, ``Incremental learning of object
  detectors without catastrophic forgetting,'' in \emph{{ICCV}}.\hskip 1em plus
  0.5em minus 0.4em\relax {IEEE} Computer Society, 2017, pp. 3420--3429.

\bibitem{Gonzalez-DiazBD19}
I.~Gonz{\'{a}}lez{-}D{\'{\i}}az, J.~Benois{-}Pineau, J.~Domenger, D.~Cattaert,
  and A.~de~Rugy, ``Perceptually-guided deep neural networks for ego-action
  prediction: Object grasping,'' \emph{Pattern Recognition}, vol.~88, pp.
  223--235, 2019.

\bibitem{Gonzalez-Diaz2018}
I.~Gonz{\'{a}}lez-D{\'{i}}az, J.~Benois-Pineau, J.-P. Domenger, and A.~de~Rugy,
  ``{Perceptually-guided Understanding of Egocentric Video Content},'' pp.
  434--441, 2018.

\bibitem{CIFAR2009}
A.~Krizhevsky, ``{Learning Multiple Layers of Features from Tiny Image},'' MIT,
  NYU, Tech. Rep., 04 2009.

\bibitem{Hafner'95}
J.~{Hafner}, H.~S. {Sawhney}, W.~{Equitz}, M.~{Flickner}, and W.~{Niblack},
  ``Efficient color histogram indexing for quadratic form distance functions,''
  \emph{IEEE Transactions on Pattern Analysis and Machine Intelligence},
  vol.~17, no.~7, pp. 729--736, 1995.

\bibitem{ZemmariB20}
A.~Zemmari and J.~Benois{-}Pineau, \emph{Deep Learning in Mining of Visual
  Content}, ser. Springer Briefs in Computer Science.\hskip 1em plus 0.5em
  minus 0.4em\relax Springer, 2020.

\bibitem{DBLP:journals/corr/SalimansK16}
\BIBentryALTinterwordspacing
T.~Salimans and D.~P. Kingma, ``Weight normalization: {A} simple
  reparameterization to accelerate training of deep neural networks,''
  \emph{CoRR}, vol. abs/1602.07868, 2016. [Online]. Available:
  \url{http://arxiv.org/abs/1602.07868}
\BIBentrySTDinterwordspacing

\bibitem{he2016identity}
K.~He, X.~Zhang, S.~Ren, and J.~Sun, ``Identity mappings in deep residual
  networks,'' in \emph{European conference on computer vision}.\hskip 1em plus
  0.5em minus 0.4em\relax Springer, 2016, pp. 630--645.

\bibitem{kingma2014adam}
D.~P. Kingma and J.~Ba, ``Adam: A method for stochastic optimization,''
  \emph{arXiv preprint arXiv:1412.6980}, 2014.

\bibitem{Yosin'2014}
J.~Yosinski, J.~Clune, Y.~Bengio, and H.~Lipson, ``{How transferable are
  features in deep neural networks?}'' in \emph{{Advances in Neural Information
  Processing Systems 27}}, Z.~Ghahramani, M.~Welling, C.~Cortes, N.~Lawrence,
  and K.~Weinberger, Eds.\hskip 1em plus 0.5em minus 0.4em\relax Curran
  Associates, Inc., 2014, pp. 3320--3328.

\bibitem{chollet2017deep}
F.~Chollet, ``Deep learning with python,'' 2017.

\end{thebibliography}

\end{document}


\maketitle
\thispagestyle{empty}
\pagestyle{empty}

\section{Appendix A}

\subsection{Mathematical justification of Move-To-Data}
In the following let $d\in \mathbb{N}$ and $\epsilon \in (0,1)$. We fix a vector $w \in \mathbb{R}^d$ and a sequence $(v^n)_{n\in \mathbb{N}} \subset \mathbb{R}^d$. The sequence resulting from the \textit{Move-To-Data} Algorithm $(w^n)_{n\in \mathbb{N}}$ is defined recursively as
$$w^{n+1} := w^n + \epsilon (v^n -  w^n) $$
for all $n\in \mathbb{N}$ with the notation $w^0 := w$.
\begin{lemma}
The \textit{Move-To-Data} Algorithm can be rewritten in 
$$w^n = \sum_{i=1}^n \epsilon(1-\epsilon)^{n-i} v^i. $$
\end{lemma}
\begin{proof}
We simply expand
\begin{align*}
w^{n+1} &= w^n + \epsilon (v^n -  w^n) \\
&= \epsilon v^n + (1-\epsilon) w^n \\
&= \epsilon v^n + (1-\epsilon)(\epsilon v^{n-1} + (1-\epsilon) w^{n-1 }) \\
\end{align*}
Then continuing the pattern
$$w^{n+1} = \sum_{i=1}^{n+1} \epsilon(1-\epsilon)^{n-i} v^i. $$
\end{proof}
In order to proof stochastic results we introduce a measure space $(\Omega , \mathcal{B}_\Omega , \mu )$